\documentclass[runningheads]{llncs}
\usepackage[T1]{fontenc}
\usepackage{graphicx}
\usepackage{times} 
\usepackage{soul}
\usepackage{url}
\usepackage[hidelinks]{hyperref}
\usepackage[utf8]{inputenc}
\usepackage[small]{caption}
\usepackage{graphicx}
\usepackage{amsmath}
\usepackage{booktabs}
\usepackage[switch]{lineno} 
\usepackage[algoruled,ruled,vlined]{algorithm2e}
\usepackage{tikz}

\usepackage{subcaption}
\usepackage{graphicx}

\usepackage{tikz}
\usetikzlibrary{positioning}

\newcommand{\bcp}{{BCP}\xspace}
\newcommand{\ps}{{${PS}$}\xspace}
\newcommand{\propps}{{$PROP_{PS}$}\xspace}
\newcommand{\bpluse}{{\bf B+E}\xspace}
\newcommand{\dfour}{{\bf d4}\xspace}
\newcommand{\decdnnf}{{\bf decdnnf\_rs}\xspace}
\newcommand{\ctwod}{{\bf C2D}\xspace}
\newcommand{\dSharp}{{\bf Dsharp}\xspace}
\newcommand{\sharpSAT}{{\bf SharpSAT-TD}\xspace}
\newcommand{\sat}{SAT\xspace}
\newcommand{\cnf}{CNF\xspace}
\newcommand{\dDNNF}{d-DNNF\xspace}

\newcommand{\var}[1]{Var (#1)\xspace}
\newcommand{\model}[1]{Mod (#1)\xspace}
\newcommand{\minI}[1]{{#1}.min\xspace}
\newcommand{\maxI}[1]{{#1}.max\xspace}
\newcommand{\pr}[1]{Pr[{#1}]\xspace}
\newcommand{\arr}[1]{[{#1}]\xspace}

\begin{document}
\title{Enhancing Query Efficiency for d-DNNF Representations Through Preprocessing}
\author{Jean Marie Lagniez\orcidID{0000-0002-6557-4115} \and Emmanuel Lonca\orcidID{0000-0002-9502-2821}}
\authorrunning{Lagniez and Lonca}
\institute{CRIL, U. Artois \& CNRS, F-62300 Lens, France\\
\email{\{lagniez,lonca\}@cril.fr}
}

\maketitle              % typeset the header of the contribution
\begin{abstract}
In this paper, we investigate preprocessing techniques aimed at improving the efficiency of accessing models of propositional formulas represented in conjunctive normal form (CNF). 
We focus on three fundamental tasks: uniform sampling, direct model access, and model enumeration. Our analysis reveals that most state-of-the-art preprocessors, when they do not preserve formula equivalence, are generally unsuitable for these tasks. 
In contrast, we demonstrate that preprocessors which preserve model counts can be effectively leveraged, provided relevant preprocessing information is maintained. 
To validate our approach, we perform extensive experiments on a diverse suite of benchmarks from multiple domains. 
The experimental results show that our preprocessing methods are both efficient and robust, yielding significant performance improvements for model access queries when CNF formulas are compiled into \dDNNF{} representations.

\end{abstract}

\section{Introduction}
Propositional logic forms the backbone of a wide array of fields, including databases~\cite{AbiteboulHV95}, automated planning~\cite{FikesN71}, and explainable artificial intelligence~\cite{Darwiche23}, among others.
When problems from these domains are encoded as propositional formulas, efficiently querying these formulas becomes essential for extracting meaningful insights and solving practical tasks.
Crucial queries in this context include \emph{model counting}, \emph{direct access}, \emph{uniform sampling}, and \emph{model enumeration}.
However, the computational complexity of these queries, often \#P-complete, poses significant challenges, making the choice of suitable formula representations highly influential to overall performance.
In this work, we focus specifically on formulas expressed in \emph{conjunctive normal form} (CNF), the predominant representation in many practical settings due to its compatibility with modern SAT solvers and widespread use in real-world applications.

To overcome the inherent computational difficulties associated with these queries on CNF formulas, \emph{preprocessing} has emerged as a vital strategy.
Preprocessing involves transforming a CNF formula into an equivalent (or, under certain relaxations, satisfiability-equivalent) CNF representation that is better suited for efficient query evaluation.
Such transformations are valuable if they facilitate downstream tasks, even after accounting for the preprocessing cost itself.
Indeed, preprocessing has proven effective across a variety of reasoning tasks, such as \sat{} solving~\cite{BiereJK21}, model counting~\cite{LagniezM17b,SoosM24}, and almost-uniform sampling~\cite{SoosM22}, often yielding significant runtime improvements.

This paper provides a comprehensive analysis of elementary preprocessing techniques, with a particular focus on their impact and suitability for direct access, uniform sampling, and model enumeration tasks.
The techniques considered include \emph{vivification}, \emph{occurrence reduction}, \emph{backbone identification}, \emph{variable elimination}, \emph{blocked clause elimination}, and the removal of \emph{implicitly} and \emph{explicitly defined variables}~\cite{LagniezM17b,LagniezLM20}.
Notably, the first three techniques preserve full logical equivalence, making them broadly applicable regardless of the specific query.
In contrast, variable elimination and blocked clause elimination only preserve satisfiability, limiting their utility for our target queries.
Finally, under certain conditions, eliminating implicitly or explicitly defined variables (while preserving the model count) can enable more efficient solutions for direct access, uniform sampling, and model enumeration.

To empirically evaluate the practical benefits of preprocessing techniques, we adopt a pipeline centered on compiling propositional formulas into \dDNNF. 
Specifically, we first apply preprocessing to the CNF formula, then compile the resulting CNF into a \dDNNF{} representation, upon which queries are subsequently answered.
This workflow reflects the common strategy for tackling the considered queries, as many inference and enumeration tasks can be performed efficiently on \dDNNF{} representations. 
Our experiments leverage a diverse suite of benchmarks, utilizing the state-of-the-art \dDNNF{} compiler \dfour~\cite{LagniezM17} in conjunction with preprocessing capabilities from \bpluse~\cite{LagniezM23}. 
We systematically assess the impact of different preprocessing strategies, ranging from no preprocessing to advanced techniques that remove explicitly defined variables, on overall query runtime and efficiency. 
Through this study, we aim to provide practitioners with clearer insight into the trade-offs between preprocessing overhead and query performance, and to illustrate how effective preprocessing can substantially enhance the practical utility of logical reasoning tools for challenging inference and enumeration tasks.

The remainder of this paper is organized as follows.
Section~\ref{sec:preliminaries} presents the necessary formal preliminaries.
Section~\ref{sec:preprocessing} introduces and analyzes the considered preprocessing techniques, focusing on their applicability to direct access, uniform sampling, and model enumeration.
Section~\ref{sec:experiments} reports our experimental results and evaluates the effectiveness of these techniques.
Finally, Section~\ref{sec:conclusion} concludes the paper and outlines promising directions for future work.

\section{Formal Preliminaries\label{sec:preliminaries}}
We consider a propositional language \propps{} in the standard manner, derived from a finite set \ps{} of propositional symbols and the 
standard logical connectives ($\land$, $\lor$, $\leftarrow$, $\leftrightarrow$, $\neg$).
\propps{} is interpreted classically.
For any formula $\Sigma$ in \propps{}, $\mathit{Var}(\Sigma)$ denotes the set of propositional variables present in $\Sigma$.
Given a finite set of variables $X$, ${\{0, 1\}}^X$ represents the set of all possible Boolean assignments to the variables in $X$.
Each propositional formula $\Sigma$ in \propps{} defines a Boolean function over $\var{\Sigma}$, 
mapping $\Sigma$ from ${\{0, 1\}}^{|\var{\Sigma}|}$ to $\{0, 1\}$. 
Assignments to $\var{\Sigma}$ that evaluate to 1 under $\Sigma$ are termed \emph{satisfying assignments} or \emph{models} of $\Sigma$.
$\model{\Sigma}$ represents the set of all models of $\Sigma$. 
Two formulas $\Sigma_1$ and $\Sigma_2$ are considered equivalent if their sets of models are identical, that is, if $\model{\Sigma_1} = \model{\Sigma_2}$. 
This equivalence is denoted as $\Sigma_1 \equiv \Sigma_2$.
$\Sigma_1$ implies $\Sigma_2$, denoted as $\Sigma_1 \models \Sigma_2$, if $\model{\Sigma_1} = \model{\Sigma_2}$.
$\bot$ represents the formula which is always falsified and $\top$ the formula which is always satisfied.

A \emph{literal} is defined as either a Boolean variable or its negation.
For any literal $\ell$, $\mathit{Var}(\ell)$ represents the variable $x$ of $\ell$ ($\mathit{Var}(x) = x$ and $\mathit{Var}(\neg x) = x$), 
and ${\sim}\ell$ denotes the complementary literal of $\ell$. In other words, for every variable $x$, ${\sim}x = \neg x$ and ${\sim}\neg x = x$. 
The conditioning of a formula $\Sigma$ by a literal $\ell = x$ (resp. $\ell = \neg x$) results in the formula $\Sigma[\ell]$, 
where each occurrence of $x$ (resp. $\neg x$) in $\Sigma$ is replaced by $\top$, and each occurrence of $\neg x$ (resp. $x$) is replaced 
by $\bot$.
After such replacement, simplification is carried out using the semantics of the logical connectors (e.g., $\top \vee \Gamma = \top$, $\top \land \Gamma = \Gamma$,
etc.) until a fixed point is reached.
This notion can be extended to set of literals $S = \{\ell_1, \ldots, \ell_m\}$ in the following way $\Sigma[S] = ((\Sigma[\ell_1]) \ldots)[\ell_m]$.
Each assignment $\mu$ is conceptualized as a (conjunctively interpreted) set of literals. 
We differentiate between total assignments and partial assignments based on whether all variables are assigned truth values or not, respectively.

A {CNF} formula $\Sigma$ is a conjunction of  clauses, where a clause is a disjunction of literals.  
Every {CNF}  is viewed  as a set  of clauses,  and every clause is  viewed as a  set of literals. 

\begin{example}\label{ex:running}
Let $\Psi = \{a \vee b, \neg a \vee \neg b, c \vee d \vee a, c \vee d \vee b\}$ a {CNF} formula and $\mathit{Var}(\Psi) = \{a,b,c,d\}$.
\end{example}

\subsection{Knowledge Compilation and d-DNNF representation}

A d-DNNF formula, which stands for Deterministic Decomposable Negation Normal Form, is a Boolean circuit with a single output, serving as its root.
It can be conceptualized as a rooted Directed Acyclic Graph (DAG), denoted as $\langle V,E \rangle$, where each input is either a literal 
or a Boolean constant ($\bot$ or $\top$), and each internal gate is either a decomposable $\land$ gate or a deterministic $\vee$ gate.
In a decomposable gate $N = \wedge(N_1, \ldots, N_k)$, no common variable is shared between the sub-circuits rooted at $N_i$ 
and $N_j$ for all $i \neq j$. 
In a deterministic gate $N = \vee(N_1, \ldots, N_k)$, the sub-circuits rooted at $N_i$ and $N_j$ are jointly inconsistent for
all $i \neq j$.
The size of a \dDNNF{} $\Sigma = \langle V,E \rangle$, denoted by $|\Sigma|$ is its number of edges $|E|$. 
\dDNNF{} is universal, as it can accommodate every propositional theory~\cite{Darwiche02a}.

Knowledge compilers such as \ctwod~\cite{Darwiche04}, \dSharp~\cite{MuiseMBH12}, 
\dfour~\cite{LagniezM17}, and \sharpSAT~\cite{KieselE23} are not able to produce \dDNNF but they are able of producing 
a sub-class of \dDNNF{} which is decision-DNNF (decision Decomposable Negation Normal Form) representations.
decision-DNNF  is defined similarly, but with decision gates of the form 
$N = \mathit{ite}(x, N_1, N_2)$ replacing deterministic $\vee$ gates. 
Here, $x$ is the decision variable at gate $N$, absent in the sub-circuits $N_1$ or $N_2$, and $\mathit{ite}$ is a ternary connective 
denoting ``{\tt if} \ldots {\tt then} \ldots {\tt else} \ldots''.
decision-DNNF representations, also termed decomposable decision graphs~\cite{FargierM06}, can be converted into specific d-DNNF representations in linear time. 
By replacing a decision node of the form $N = \mathit{ite}(x, N_1, N_2)$ in a decision-DNNF representation with 
$N = (\neg x \wedge N_1) \vee (x \wedge N_2)$, the resulting d-DNNF representation maintains decomposable $\land$ nodes 
(as $x$ appears neither in $N_1$ nor in $N_2$) and a deterministic $\vee$ node (since $(\neg x \wedge N_1) \wedge (x \wedge N_2)$ is 
inconsistent).
For simplicity of exposition, we will use the term \dDNNF{} throughout the remainder of this paper, although our experiments are carried out using decision-DNNF representations. Importantly, all our results apply equally to both \dDNNF{} and decision-DNNF formats.

\begin{example}[Example~\ref{ex:running} cont’ed]
Consider the \cnf{} formula $\Psi$ given in Example~\ref{ex:running}, the \dDNNF{} $\Sigma = ((\neg a \wedge b) \vee (a \wedge \neg b)) 
\wedge (c \vee (\neg c \wedge d))$ is equivalent to $\Psi$.
\end{example}

\dDNNF{} serves as a compelling language of representation due to its ability to efficiently handle various queries and transformations, 
such as satisfiability and conditioning in polynomial time. 
Notably, queries involving direct access~\cite{CarmeliTGKR23,BringmannCM22}, uniform sampling~\cite{SharmaGRM18} and model enumeration~\cite{Darwiche02a}, 
which we will discuss in the next section, can be answered efficiently when the formula is represented as a \dDNNF.
Although our experimental methodology is based on compiling CNF formulas into \dDNNF{}, the applicability of our findings regarding suitable preprocessing techniques extends beyond this specific approach. 
Indeed, these results hold for any query-strategy, as preprocessing can be applied prior to query evaluation, and the answers obtained for the simplified formula can subsequently be mapped back to the original CNF in polynomial time.
This generality underscores the practical value of our recommended preprocessing techniques, regardless of the downstream reasoning or enumeration strategy.

\subsection{Model Enumeration, Direct Access and Uniform Sampling Queries}

The enumeration problem involves listing the set of models of a propositional formula without redundancies, commonly referred to as the disjoint AllSAT problem. 
Models can be enumerated in two forms: \emph{complete} or \emph{partial}. 
A \emph{complete model} assigns a value to every propositional variable in the formula. 
However, due to the often vast number of complete models associated with a formula, compact representations are desirable in certain applications.
A \emph{partial model} provides such a compact representation by allowing some variables to remain unassigned. 
For a partial model to be valid, it must ensure that assigning any truth value to the unassigned variables does not alter the satisfiability of the model. 
Consequently, a partial model with $m$ assigned variables represents $2^{n-m}$ complete models, where $n$ is the total number of variables.

As we will demonstrate in the next section, the preprocessing techniques applicable to these two enumeration tasks behave differently and require distinct considerations.

\begin{example}[Example~\ref{ex:running} cont’ed]\label{ex:firstModelEnum}
    For the \cnf{} formula $\Phi$ provided in Example~\ref{ex:running}, the complete models of $\model{\Psi}$ are:  
    $\{a,\neg b,c,d\}$, $\{\neg a, b,c,d\}$, $\{a,\neg b,c, \neg d\}$
    $\{\neg a, b,c, \neg d\}$, $\{a,\neg b,\neg c,d\}$, $\{\neg a, b, \neg c,d\}$.
    
    A possible compact representation of $\Phi$ can be expressed using the following partial models:  
    $\{a,\neg b,c\}$, $\{\neg a, b,c\}$, $\{a,\neg b,\neg c,d\}$, $\{\neg a, b, \neg c,d\}$.
    
\end{example}

The direct access task, first introduced in the database context by~\cite{BaganDGO08}, consists of, given an input $k$ and an order $\prec_{lex}$ over 
the assignments, returning the $k$-th model of a propositional formula $\Phi$ with respect to $\prec_{lex}$ if $k \leq |\model{\Phi}|$ and failing otherwise. 
In the case of propositional logic, $\prec_{lex}$ involves fixing an order $\tau$ on the Boolean variables and then considering each assignment as a 
word constructed from ${\{0,1\}}^{|\var{\Phi}|}$.
We will denote $\prec_{lex}^{\tau}$ when $\prec_{lex}$ depends on the order $\tau$. 
We will use the notation $<_{\tau}$ to specify that a variable $x$ precedes a variable $y$ according to $\tau$, i.e., $x <_{\tau} y$.

\begin{example}[Example~\ref{ex:running} cont’ed]\label{ex:firstDirectAccess}
Given the \cnf{} formula $\Phi$ in Example~\ref{ex:running}, if $\tau = (a,b,c,d)$, then the first 
model of $\Phi$ is $\{\neg a, b, \neg c, d\}$ (corresponding to the word $(0101)$) and the third model is $\{\neg a, b, c, d\}$ (corresponding 
to the word $(0111)$). 
If instead we use the variable ordering $\tau = (d, c, b, a)$, the first model becomes $\{a, \neg b, c, \neg d\}$ (bitstring $(0101)$), and the last model is $\{\neg a, b, c, d\}$ (bitstring $(1110)$).
\end{example}

The direct access query can serve as a building block for many other important tasks, such as counting, enumerating, and sampling without 
repetition~\cite{SharmaGRM18}.
However, answering this query is generally hard (\#P-difficult) for propositional logic. 
Nevertheless, since \dDNNF{} supports both conditioning and model counting in polynomial time, this query can be answered in polynomial 
time.
Specifically, if $|\model{\Phi}| \leq k$, an error is returned. Otherwise, starting with an empty term $\sigma$ and an interval 
$I = [0, |\model{\Phi}|]$, the algorithm iteratively picks the next variable $x$ in order. 
If $\minI{I} + |\model{\Phi\arr{\sigma}\arr{\neg x}}| \geq k$, then $\sigma = \sigma \cup \{\neg x\}$ and $I$ is updated to $[\minI{I}, \minI{I} + |\model{\Phi\arr{\sigma}\arr{\neg x}}|]$. 
Otherwise, $\sigma = \sigma \cup \{x\}$ and $I$ is updated to $[\minI{I} + |\model{\Phi\arr{\sigma}\arr{\neg x}}| + 1, \maxI{I}]$, where $\minI{I}$ is the 
lower endpoint and $\maxI{I}$ is the upper endpoint.

\begin{example}[Example~\ref{ex:firstDirectAccess} cont’ed]
Let us consider the previous example when we seek the third model. The algorithm starts with $\sigma = \emptyset$ and $I = [0, 6]$. 
As we can see, $|\model{\Phi[\neg a]}| = 3$, which means $\neg a$ is added to $\sigma$ and $I$ is updated to $[0, 3]$. 
Next, $b$ is considered, and $|\model{\Phi[\{\neg a,\neg b\}]}|$ is computed to be $0$. 
Since $0 + 0 < 3$, $b$ is added to $\sigma$ and $I$ remains $[0, 3]$. 
Then, $c$ is selected and $|\model{\Phi[\{\neg a,b,\neg c\}]}|$ is computed to be $1$. 
As $0 + 1 < 3$, $c$ is added to $\sigma$ and $I$ becomes $[1, 3]$. 
Finally, $d$ is picked, and since $|\model{\Phi[\{\neg a,b,c,\neg d\}]}| = 1$ and $1 + 1 < 3$, $\sigma$ is ultimately equal to 
$\{\neg a, b, c, d\}$, which is the expected result of $I$.
\end{example}

The concept of uniform sampling involves generating samples $\mathcal{R}_{\Phi}$ from the set of models of $\Phi$ using a generator $\mathcal{G}$ that 
ensures $\forall \mu \in \mathcal{R}_{\Phi}, \pr{\mathcal{G}(\Phi) = \mu} = 1/|\model{\Phi}|$.
A technique for uniform sampling involves these steps: first, count the total number of models, $c$. 
Then, generate $k$ integers uniformly within the range $\{1, \ldots, c\}$ and use direct access to identify the corresponding models.

\begin{example}[Example~\ref{ex:running} cont’ed]
Given the \cnf{} formula $\Phi$ in Example~\ref{ex:running}, we first compute its total number of models, which is $6$, before selecting 2 random models.
Next, we randomly choose a set of two integers from the set $\{1, \ldots, 6\}$; for example, $\{1, 4\}$. 
Finally, if the order considered for the direct access query is the lexicographical order of the variables in $\Phi$, 
then the set of models $\{\{\neg a, b, \neg c, d\}, \{a, \neg b, \neg c, d\}\}$ is returned.
\end{example}

In~\cite{JerrumVV86}, the authors observed a deep relationship between model counting and uniform sampling. 
They showed that given access to an exact model counter, it is possible to design a uniform generator that requires only polynomially many 
queries to the exact model counter. 
Specifically, since \dDNNF{} supports model counting queries in polynomial time, it is evident that this language is well-suited for 
sampling a set of models.
In~\cite{SharmaGRM18}, the authors propose an approach that tags the \dDNNF{} circuit to efficiently compute uniform sampling. 
This approach leverages recent advances in knowledge compilation, which can be harnessed to design a scalable uniform sampler.

It is important to note that the approach proposed in~\cite{SharmaGRM18} does not impose any order on the models, meaning that the uniform sampling generated cannot be controlled by a seed. 
More specifically, the set of models produced depends on the \dDNNF{} formula generated by the compiler, and in general, compilers do not allow for control over the compilation process. 
In the following, we will show that the selection of preprocessing techniques may vary depending on whether or not we require the ability to control the set of models through a seed.

\section{Preprocessing for Direct Access, Uniform Sampling, and Model Enumeration\label{sec:preprocessing}}
Preprocessing a propositional formula transforms it while preserving properties like satisfiability and model count.
This process is beneficial because the problem at hand (e.g., satisfiability) can often be solved more efficiently after 
the input formula has been preprocessed, accounting for the preprocessing time in the overall solving time. 
Various preprocessing techniques are now recognized as valuable for SAT solving, QBF solving, and model counting~\cite{LagniezM17,LagniezLM20,SoosM22,HeuleSB17,BiereJK21}.
These techniques can be categorized based on the type of equivalence they maintain with the input formula: satisfiability, model counting, 
or logical equivalence.

First, let us examine the preprocessing techniques that produce a formula equivalent to the input formula, such as vivification, backbone detection, and occurrence 
elimination~\cite{PietteHS08,LagniezM14}.
The backbone of a \cnf{} formula $\Phi$ is the set of all literals implied by $\Phi$ when $\Phi$ is satisfiable; 
if $\Phi$ is unsatisfiable, the backbone is the empty set. 
The purpose of backbone identification is to explicitly identify the backbone of the input \cnf{} formula $\Phi$ 
and conjoin it to $\Phi$.
Vivification~\cite{PietteHS08} is a preprocessing technique aimed at reducing a given \cnf{} formula $\Phi$ by removing 
some clauses and literals while preserving equivalence. 
Given a clause $\alpha = \ell_1 \vee \ldots \vee \ell_k$ in $\Phi$, two rules are used to determine whether $\alpha$ can be removed 
from $\Phi$ or simply shortened.
On one hand, if for any $j \in \{1, \ldots, k\}$, a Boolean Constraint Propagator (BCP)~\cite{ZhangMMM01} can prove that
$\Phi \setminus \{\alpha\} \models \ell_1 \vee \ldots \vee \ell_j$, then $\alpha$ is entailed by $\Phi \setminus \{\alpha\}$ and 
can be removed from $\Phi$. 
On the other hand, if BCP can prove that $\Phi \setminus \{\alpha\} \models \ell_1 \vee \ldots \vee \ell_j \vee {\sim}\ell_{j+1}$, 
then $\ell_{j+1}$ can be removed from $\alpha$ without affecting equivalence. 
Occurrence elimination considers only the second rule, focusing on the elimination of literals instead of clauses.

\begin{example}[Example~\ref{ex:running} cont'ed]
Let us consider again the \cnf{} formula $\Phi$ given in Example~\ref{ex:running}. 
When we observe the models, we see that the variables can take all possible values, indicating that the backbone of $\Phi$ is empty. 
Upon employing occurrence elimination, the third clause can be simplified to $c \vee d$, 
as $\Phi \setminus \{c \vee d \vee a\} \models c \vee d \vee \neg a$, provable via \bcp.
$\Phi$ is now equal to $\{a \vee b, \neg a \vee \neg b, c \vee d, c \vee d \vee b\}$. 
Now, we can see that the fourth clause can be removed using the vivification rule since $\Phi \setminus \{c \vee d \vee b\}$ 
can infer $c \vee d$ by \bcp.
\end{example}

Preprocessing techniques that ensure equivalence can be applied to a variety of queries, such as direct access, uniform sampling, 
and both complete and partial model enumeration.
However, not all techniques share this versatility.
Prominent SAT-solving preprocessing methods~\cite{LagniezM14,LagniezLM16} are unsuitable for model counting as they 
may alter the number of models in equisatisfiable formulas.
Now, let us explore the effects of preprocessing techniques solely maintaining satisfiability, like blocked clause elimination 
and variable elimination~\cite{JarvisaloBH10,SubbarayanP04}, on the queries under consideration.

The resolution rule asserts that, given two clauses $\alpha_1 = \{\ell, a_1, \ldots, a_n\}$ and 
$\alpha_2 = \{{\sim}\ell, b_1, \ldots, b_m\}$, the resulting clause $\alpha = \{a_1, \ldots, a_n, b_1, \ldots, b_m\}$ is the 
\emph{resolvent} of $\alpha_1$ and $\alpha_2$ on the literal $\ell$.
Variable elimination of $x$ in $\Phi$ is performed by removing from $\Phi$ all clauses containing the variable $x$ 
(either as a positive or negative literal) and adding into $\Phi$ all possible resolvents between the removed clauses. 
It is equivalent to existentially quantify $x$ in $\Phi$.
This rule is applied effectively only when it does not increase the number of clauses in $\Phi$~\cite{SubbarayanP04}.
The simplification technique known as \emph{blocked clause elimination}~\cite{JarvisaloBH10} targets the removal of specific 
clauses, termed \emph{blocked clauses}, from \cnf{} formulas. 
A literal $\ell$ within a clause $\alpha$ is termed a \emph{blocking literal} if it blocks $\alpha$ with respect to $\Phi$. 
This occurs when, for every clause $\alpha'$ in $\Phi$ containing ${\sim}\ell$, the resulting resolvent on $\ell$ is a tautology.
Essentially, a clause is considered blocked if it contains a literal that can effectively block it.
Applying blocked clause elimination to $\Phi$ involves removing every clause containing a blocking literal and repeating the process 
iteratively until no blocked literals remain.

\begin{proposition}
Let $\Phi$ be a \cnf{} formula. Variable elimination and blocked clause elimination cannot be applied when the objective is to answer 
direct access, uniform sampling and model enumeration queries.
\end{proposition}

\begin{proof}
Consider the \cnf{} formula $\Phi = \{a \vee b\}$ over the variable set $X = \{a, b\}$. 
It is evident that both preprocessing techniques produce the formula $\Phi' = \top$. 
From $\Phi'$ it is impossible to retrieve the models of $\Phi$. 
Thus, these preprocessing techniques are unsuitable for the queries considered.
\end{proof}

These findings echo conclusions drawn from the model counting domain, revealing the inadequacy of these preprocessing methods for addressing such problems. 
This underscores the interrelationship between model counting, direct access, and uniform sampling queries. 
To delve deeper into this relationship, let us investigate preprocessing techniques tailored to the model counting task, 
like eliminating defined variables~\cite{LagniezLM16,LagniezM14}.
Defined variables $\mathcal{O}$ of $\Phi$ are variables whose valuation depends on other variables from $Var(\Phi) \setminus \mathcal{O}$.
In propositional logic, definability can manifest in two equivalent forms: implicit and explicit.
More precisely, the formula $\Phi$ \emph{implicitly defines} the variable $y$ in terms of $X \subset Var(\Phi)$ if and only if 
for assignment $\gamma_X$ over $X$, we have $\gamma_X \wedge \Phi \models y$ or $\gamma_X \wedge \Phi \models \neg y$. 
The formula $\Phi$ \emph{explicitly defines} the variable $y$ in terms of $X$ if and only if there exists a formula 
$\Psi_X \in PROP_X$ such that $\Phi \models (\Psi_X \leftrightarrow y)$. 
In such a case, $\Psi_X$ is called a \emph{definition (or gate) of $y$ on $X$ in $\Phi$}, $y$ is the \emph{output variable} 
of the gate, and $X$ are its \emph{input variables}. 
In~\cite{LagniezLM16}, the authors demonstrate that defined variables can be eliminated while keeping the number of models unchanged.

\begin{example}[Example~\ref{ex:running} cont'ed]\label{ex:elimDef}
Let $\Phi$ the \cnf{} formula given in Example~\ref{ex:running}.
As we can see, $a \leftrightarrow \neg b$, meaning that $\Phi$ defines the variable $a$ in terms of $\{b\}$. 
After eliminating $a$, the resulting formula is $\Phi' = \{c \vee d \vee \neg b, c \vee d \vee b\} \equiv \{c \vee d\}$, defined over the Boolean variables 
$\{b, c, d\}$. 
The list of models of $\Phi'$ is $\{\{\neg b, c, d\}, \{b, c, d\}, \{\neg b, c, \neg d\}, \{b, c, \neg d\}$, $\{\neg b, \neg c, d\}$, $\{b, \neg c, d\}\}$. 
As we can see, the model count of $\Phi'$, which is $6$, is the same as that of $\Phi$.
\end{example}

While preserving the number of models, as stated in Proposition~\ref{prop:notDirectSharpEquiv}, such preprocessing techniques cannot be applied 
directly when the targeted queries are direct access, uniform sampling and model enumeration.

\begin{proposition}\label{prop:notDirectSharpEquiv}
Let $\Phi$ be a \cnf{} formula and $x$ a variable such that $\Phi$ defines $x$ in terms of $Var(\Phi) \setminus \{x\}$. 
The formula obtained after eliminating $x$ from $\Phi$ cannot be used solely to answer either direct access, uniform sampling or model enumeration queries.
\end{proposition}

\begin{proof}
Consider the \cnf{} formula $\Phi = (\neg x \vee y_1 \vee \ldots \vee y_k) \wedge (\bigwedge_{i=1}^k (x \vee \neg y_i))$, which represents the Boolean 
function $x \leftrightarrow y_1 \vee \ldots \vee y_k$. 
Clearly, $\Phi$ defines $x$ in terms of $\{y_1, \ldots, y_k\}$. 
Eliminating the variable $x$ produces the \cnf{} formula $\Phi' = \top$ over the variables $\{y_1, \ldots, y_k\}$. 
Although $\Phi$ and $\Phi'$ have the same number of models, the models of $\Phi$ cannot be obtained from $\Phi'$ alone.
\end{proof}

Proposition~\ref{prop:notDirectSharpEquiv} demonstrates that, in the general case, it is impossible to answer direct access or uniform sampling queries 
on the preprocessed formula $\Phi'$ to obtain the results for the original formula $\Phi$.
However, as demonstrated in Proposition~\ref{prop:withDefUniSampling}, this preprocessing can be utilized for uniform sampling and complete model enumeration 
queries, provided that the values of the eliminated variables can be determined from the variables in $\Phi'$.
But first, let us introduce the concept of a compatible evaluation function.

\begin{definition}[Compatible Evaluation Function]
Let $\Phi$ be a \cnf{} formula, and let $\mathcal{O} \subseteq Var(\Phi)$ be a set of variables that $\Phi$ defines in terms of 
$\mathcal{I} = Var(\Phi) \setminus \mathcal{O}$. 
$f$ is called a {\em compatible evaluation function}, if given any $o \in \mathcal{O}$ and $\gamma \in {\{0,1\}}^{\mathcal{I}}$, it 
computes the literal associated with $o$ given $\gamma$ regarding $\Phi$
(which is $o$ if $\Phi$ forces $o$ to be true given $\gamma$, and $\neg o$ otherwise). 
\end{definition}

\begin{proposition}\label{prop:withDefUniSampling}
    Let $\Phi$ be a \cnf{} formula, let $\mathcal{O} \subseteq Var(\Phi)$ be a set of variables defined by $\Phi$ in terms of 
    $\mathcal{I} = Var(\Phi) \setminus \mathcal{O}$, and let $f$ be the associated compatible evaluation function.  
    Let $\Phi'$ be the formula obtained by eliminating the variables in $\mathcal{O}$ from $\Phi$.  
    The following results hold:
    \begin{enumerate}
        \item If $\mathcal{R}_{\Phi'}$ is a uniform sample of $\Phi'$, 
        then $\mathcal{R}_{\Phi} = \{\omega \cup \bigcup_{o \in \mathcal{O}} f(o, \omega) \mid \omega \in \mathcal{R}_{\Phi'}\}$ is a 
        uniform sample of $\Phi$;
        \item $\model{\Phi} = \{\omega \cup \bigcup_{o \in \mathcal{O}} f(o, \omega) \mid \omega \in \model{\Phi'}\}$.
    \end{enumerate}
    \end{proposition}
    
\begin{proof}
    Proposition~1 in~\cite{LagniezLM20} establishes that $|\model{\Phi}| = |\model{\Phi'}|$. 
    The proof relies on showing that if $\omega \in \model{\Phi'}$, then $\omega \cup \bigcup_{o \in \mathcal{O}} f(o, \omega) \in \model{\Phi}$. 
    This implies that $f$ defines a bijection between $\model{\Phi'}$ and $\model{\Phi}$.
    Given that $f$ is bijective, the results follow directly:
    \begin{enumerate}
        \item Since $\mathcal{R}_{\Phi'}$ is a uniform sample of $\Phi'$, applying $f$ preserves uniformity. Thus, 
        $\mathcal{R}_{\Phi} = \{\omega \cup \bigcup_{o \in \mathcal{O}} f(o, \omega) \mid \omega \in \mathcal{R}_{\Phi'}\}$ is a uniform sample of $\Phi$.
        \item The bijection guarantees that every model of $\Phi'$ is transformed into a unique model of $\Phi$, ensuring 
        $\model{\Phi} = \{\omega \cup \bigcup_{o \in \mathcal{O}} f(o, \omega) \mid \omega \in \model{\Phi'}\}$.
    \end{enumerate}
\end{proof}

Observe that if $f$ can be computed in polynomial time, it facilitates the polynomial-time construction of both uniform samples and the complete models of $\Phi$ from $\Phi'$.  
This is particularly applicable when the eliminated variables are explicitly defined by Boolean functions, such as equivalence gates, AND-gates, OR-gates, or XOR-gates. 

Now, let us consider the scenario where the direct access query is the focus. 
As demonstrated in the following example, merely knowing the function $f$ is insufficient for effectively answering this query.

\begin{example}[Example~\ref{ex:running} cont'ed]
    Let us revisit the \cnf{} formula $\Phi'$ computed in Example~\ref{ex:elimDef}, which is derived from the original \cnf{} formula $\Phi$ introduced 
    in Example~\ref{ex:running}.  
    Considering the natural order, the first model of $\Phi'$, $\{\neg b, \neg c, d\}$, is augmented using the information from $f$, 
    resulting in $\{a, \neg b, \neg c, d\}$ (due to $a \leftrightarrow \neg b$).  
    However, this differs from the first model of $\Phi$, which is $\{\neg a, b, \neg c, d\}$. 
\end{example}

The problem arises because variable valuation can depend on variables that come later in the order \(\tau\) specified by \(\prec_{lex}^{\tau}\) for the direct 
access task. 
To overcome this, we define a \emph{compatible order} as follows:

\begin{definition}[Compatible order]
Let $\Phi$ be a CNF formula, and let $\mathcal{O} \subseteq Var(\Phi)$ be a set of variables defined by $\Phi$ in terms of 
$\mathcal{I} = Var(\Phi) \setminus \mathcal{O}$. 
An order $\tau$ over $\langle \Phi, \mathcal{I}, \mathcal{O} \rangle$ is  a compatible order if and only if $\tau$ is an order over $Var(\Phi)$, 
$(\mathcal{I}, \mathcal{O})$ is a bi-partition over $Var(\Phi)$, and for every $o \in \mathcal{O}$, $x <_{\tau} o$ for all $x \in \mathcal{I}_o$, 
where $\Phi$ defines $o$ in terms of $\mathcal{I}_o$.
\end{definition}

Two strategies can be applied regarding the control the user has regarding the choice of $\tau$.
The first strategy involves adjusting the order to comply with the preprocessing step, meaning all eliminated variables must be placed at the end of $\tau$. 
For example, in the previous case, $\tau = (b, c, d, a)$ is a compatible order.
The second strategy involves constraining the preprocessing method to only eliminate a variable $o$ if $\Phi$ defines $o$ in terms of $\mathcal{I}$, such that for 
all $x \in \mathcal{I}$, $x <_{\tau}o$. 
For example, in the previous case, $\tau = (b, a, c, d)$ is a compatible order.

We can observe that the first strategy allows for the elimination of more variables during the preprocessing step but restricts the operator 
$\prec_{lex}^{\tau}$ by partially fixing it. 
Conversely, the second strategy may reduce the number of eliminated variables, but it grants the user complete freedom in choosing $\prec_{lex}^{\tau}$. 
Regardless of the chosen strategy, Proposition~\ref{prop:directAccessOrdered} demonstrates that it is possible to leverage preprocessing methods that eliminate
defined variables to construct a \cnf{} formula $\Phi'$ from the input \cnf{} formula $\Phi$, 
while still maintaining the ability to reason about $\Phi'$ to answer the direct access query on $\Phi$.

\begin{proposition}\label{prop:directAccessOrdered}
Let $\Phi$ be a CNF formula, let $\mathcal{O} \subseteq Var(\Phi)$ be a set of variables such that $\Phi$ defines $\mathcal{O}$ in terms of 
$\mathcal{I} = Var(\Phi) \setminus \mathcal{O}$, and $f$ the associated compatible evaluation function.
Let $\Phi' = \exists \mathcal{O}.\Phi$ and $\tau$ is a compatible order over $\langle \Phi, \mathcal{I}, \mathcal{O} \rangle$.
If $\omega'$ is the $k$-th model of $\Phi'$ with respect to $\prec_{lex}^{\tau'}$ (where $\tau'$ is the projection of $\tau$ onto $Var(\Phi')$), 
then $\omega = \omega' \cup \bigcup_{o \in \mathcal{O}} f(\mathcal{I}_o, o)$ will be the $k$-th model of $\Phi$.
\end{proposition}

\begin{proof}
First, since the preprocessing technique preserves the number of models, if $k > |Mod(\Phi')|$, then $k > |Mod(\Phi)|$. 
Therefore, the answer to the direct access query remains the same whether it is applied to $\Phi'$ or $\Phi$, meaning that it will fail for both 
formulas as expected.
Now let us suppose that $k \leq |Mod(\Phi')|$.

First, let us demonstrate that if $\omega_1, \omega_2 \in Mod(\Phi)$ such that $\omega_1 \prec_{lex}^{\tau} \omega_2$, 
then $\omega_1' \prec_{lex}^{\tau'} \omega_2'$.
As $\Phi' = \exists \mathcal{O}.\Phi$, we directly deduce that $\omega_1'$ and $\omega_2'$ are models of $\Phi'$.
Now, let us argue by contradiction. Suppose $\omega_2' \prec_{lex}^{\tau'} \omega_1'$. 
By the definition of $\prec_{lex}^{\tau'}$, this implies that there exists $\neg x \in \omega_2'$ such that $x \in \omega_1'$, and 
$\forall \ell \in \omega_2'$ with $Var(\ell) <_{\tau'} x$, we have $\ell \in \omega_1'$.
Since $\omega_1'$ and $\omega_2'$ are projections of $\omega_1$ and $\omega_2$ onto $Var(\Phi')$, it follows that $\neg x \in \omega_2$ and $x \in \omega_1$.
Moreover, for all literals $\ell \in \omega_1$ with $Var(\ell) <_{\tau} x$, we have $\ell \in \omega_2$. 
This is because if $Var(\ell) \in \mathcal{I}$, then $\ell$ is trivially in $\omega_2$. 
Otherwise, since $\tau$ is a compatible order, all variables in $\mathcal{O}$ will be fixed with respect to $f$ in both $\omega_1$ and $\omega_2$. 
Since $\neg x \in \omega_2$ and $x \in \omega_1$, this implies $\omega_2 \prec_{lex}^{\tau} \omega_1$, 
contradicting the initial assumption that $\omega_1 \prec_{lex}^{\tau} \omega_2$.

Then, let us demonstrate that if $\omega_1', \omega_2' \in Mod(\Phi')$ such that $\omega_1' \prec_{lex}^{\tau'} \omega_2'$, 
then $\omega_1 \prec_{lex}^{\tau} \omega_2$ with $\omega_1 = \omega_1' \cup \bigcup_{o\in\mathcal{O}} f(o, \omega_1')$ and 
$\omega_2 = \omega_2' \cup \bigcup_{o\in\mathcal{O}} f(o, \omega_2')$. 
The proof is straightforward and follows from the fact that $\tau$ is a compatible order.

To conclude this proof, let us demonstrate that if $\omega'$ is the $k$-th model of $\Phi'$ with respect to $\prec_{lex}^{\tau'}$, 
then $\omega = \omega' \cup \bigcup_{o \in \mathcal{O}} f(o, \omega')$, will be the $k$-th model of $\Phi$. 
We will argue by contradiction and assume that $\omega$ is not the $k$-th model of $\Phi$. 
This means $\omega$ is the $j$-th model of $\Phi$ where $j < k$ or $j > k$.
First, consider the case where $j < k$. 
Let $\Omega' \subseteq Mod(\Phi')$ be the first $k$ models of $\Phi'$. 
By construction, all models in $\Omega'$ are disjoint, and by the definition of definability, each corresponds to exactly one model of $\Phi$. 
Since $\omega_1' \prec_{lex}^{\tau'} \omega_2'$ implies $\omega_1 \prec_{lex}^{\tau} \omega_2$, the $k$-th model of $\Phi'$ cannot be associated 
with the $j$-th model of $\Phi$ where $j < k$, otherwise $|\Omega'| < k$.
Now, consider the case where $j > k$. 
Let $\Omega \subseteq Mod(\Phi)$ be the first $j-1$ models of $\Phi$ with respect to $\tau$. 
Since $\tau$ is a compatible order, particularly $\Phi$ defines $\mathcal{O}$ in terms of $\mathcal{I}$, 
for all $\omega_1, \omega_2 \in \Omega$, the projected models $\omega_1'$ and $\omega_2'$ of $\omega_1$ 
and $\omega_2$ over $Var(\Phi')$ are such that $\omega_1' \neq \omega_2'$. 
Since $\omega_1 \prec_{lex}^{\tau} \omega_2$ implies $\omega_1' \prec_{lex}^{\tau'} \omega_2'$, there exist at least $k$ models $\Omega''$ of $\Phi'$ 
such that $\omega'' \prec_{lex} \omega'$ with $\omega'' \in \Omega''$. 
This contradicts the fact that $\omega'$ is the $k$-th model of $\Phi'$.
\end{proof}

It is important to note that, by leveraging the previous proposition, it is also possible to design a pseudo-uniform sampler that incorporates 
a seed as part of its operation.  
This capability is particularly valuable when determinism is required, as the seed allows the sampling process to be reproducible 
and consistent across different runs.  
Such a pseudo-uniform sampler can be employed in scenarios where repeatability is essential, such as debugging, benchmarking, 
or ensuring fairness in randomized processes.

In conclusion, when enumerating partial models, it is important to note that eliminating defined variables during preprocessing can create an exponential gap between the number of partial models in the original and simplified formulas.
This is demonstrated in the following example, where the preprocessing step significantly reduces the formula but at the cost of losing 
information critical for partial model enumeration.  

\begin{example}
Consider the \cnf{} formula $\Phi$ representing the Boolean function $x \leftrightarrow \bigoplus_{i=1}^n y_i$. 
In this function, the variable $x$ is fully determined by the parity of the variables $\{y_1, \dots, y_n\}$. 
As such, during the preprocessing phase, $x$ can be safely eliminated, resulting in the simplified formula 
$\Phi' = \exists x.\Phi = \top$. 
A compact representation of $\Phi'$ is $\{\top\}$. However, this representation no longer encodes any partial models of the original formula $\Phi$. 
Moreover, it is well known that the parity function $\neg x \oplus \bigoplus_{i=1}^n y_i$ does not admit a compact representation. 
Thus, there is an exponential gap between the number of partial models of $\Phi$ and those of $\Phi'$.
\end{example}

\section{Experimentation\label{sec:experiments}}
To systematically evaluate the impact of the preprocessing methods outlined above on the efficiency of answering direct access, uniform sampling, and model enumeration queries, we conducted a series of experiments focused on runtime performance and scalability. The knowledge compiler \dfour{} (\url{https://github.com/crillab/d4v2}) served as the backbone for the \dDNNF{} compilation process, while preprocessing strategies were inspired and implemented based on features available in the \bpluse{} framework (\url{https://github.com/crillab/b-plus-e})~\cite{LagniezM23}. Our work also introduced new preprocessing techniques that emphasize fixed variable orderings, extending the capabilities of existing tools.

\medskip
\noindent
\textbf{Implementation Pipeline.}
Our experimental pipeline, presented in Figure~\ref{fig:pipeline},  comprises the following stages:
\begin{itemize}
\item \textbf{Preprocessing:} The \bpluse~\cite{LagniezM23} preprocessor is used to apply a suite of simplifications, including variable elimination and redundancy detection. Depending on the experimental setting, \bpluse{} is configured to operate in equivalence-preserving or model count-preserving modes, optionally enforcing a fixed variable ordering.
\item \textbf{Knowledge Compilation:} The simplified CNF formula is compiled into a decision-DNNF using the \dfour~\cite{LagniezM17} compiler. As previously discussed, our results apply to both decision-DNNF and general \dDNNF{} representations.
\item \textbf{Query Answering:} For each compiled circuit, we answer direct access, uniform sampling, and model enumeration queries using the \decdnnf{} tool (\url{https://crates.io/crates/decdnnf_rs})~\cite{LagniezL24}, building on customized enumeration modules.
\end{itemize}

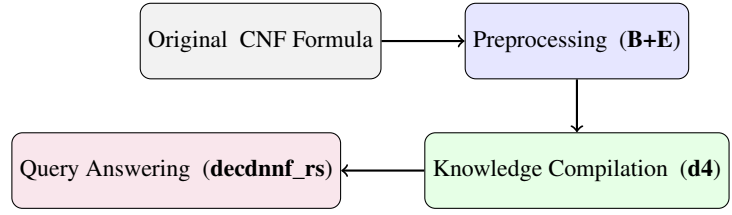
\begin{figure}[t]
\centering
\begin{tikzpicture}[
node distance=7mm and 11mm,
every node/.style={font=\small},
block/.style={draw, rectangle, rounded corners, minimum width=25mm, minimum height=10mm, align=center}
]
% Nodes
\node[block, fill=gray!10]   (cnf)          {Original \ CNF Formula};
\node[block, fill=blue!10, right=of cnf]    (preproc)      {Preprocessing \ (\bpluse)};
\node[block, fill=green!10, below=of preproc] (compile)      {Knowledge Compilation \ (\dfour)};
\node[block, fill=purple!10, left=of compile] (queries)     {Query Answering \ (\decdnnf)};
% Arrows
\draw[->, thick] (cnf) -- (preproc);
\draw[->, thick] (preproc) -- (compile);
\draw[->, thick] (compile) -- (queries);
\end{tikzpicture}
\caption{\label{fig:pipeline}Experimental pipeline for preprocessing, compilation, and query answering.}
\end{figure}

\noindent
\textbf{Preprocessing Strategies.}
We evaluated the following preprocessing configurations in our experiments:
\begin{itemize}
\item \texttt{no}: \dfour{} is run directly on the original formula, without preprocessing.
\item \texttt{equiv}: The formula is preprocessed using vivification, backbone detection, and occurrence elimination.
\item \texttt{\#equiv-explicit}: In addition to the \texttt{equiv} preprocessing, explicitly defined variables are eliminated.
\item \texttt{\#equiv-explicit-ordered}: Builds on \texttt{\#equiv-explicit} but enforces a compatible variable ordering $\tau$; for our experiments, we select the natural order.
%\item \texttt{#equiv-implicit}: The formula is preprocessed with \texttt{equiv} and the elimination of implicitly defined variables before compilation.
\end{itemize}
These strategies extend and customize the core functionalities of \bpluse{} to suit the requirements of our workflow.

\medskip
\noindent
\textbf{Experimental Setup.}
All experiments were conducted on a cluster with dual quad-core Intel Xeon E5-2637 v4 CPUs (3.50~GHz)., 128GiB RAM, and running CentOS 8 (kernel 4.18.0-301.1.el8.x86\_64). 
Hyperthreading was disabled and no cache sharing between cores was permitted. 
Each run was constrained to 3,600 seconds of CPU time and 32GiB of memory. Compilation was performed with \texttt{g++} version 13.2.0. 
We evaluated 1,425 benchmark instances from previous uniform sampling studies\cite{SharmaGRM18}, available at \url{https://github.com/meelgroup/KUS}.
All experimental logs and materials needed for reproducibility are provided at \url{https://zenodo.org/records/15837216}.

\begin{figure}[t]
\centering
\includegraphics[scale=0.55]{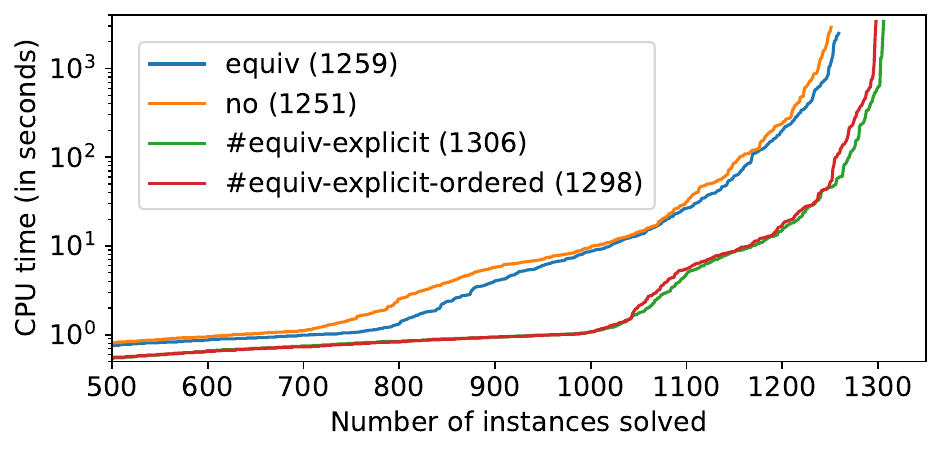}
\caption{\label{fig:times}
Cactus plot showing the running time of \dfour{} with various preprocessing methods. Each line represents a different preprocessing method, with the number of instances solved indicated in parentheses in the legend. The plot displays the number of instances completed within a given CPU time limit, measured in seconds. }
\end{figure}

\medskip
\noindent
\textbf{Results.}
Figure~\ref{fig:times} presents a cactus plot that illustrates the impact of each preprocessing method on the performance of \dfour.
This plot also shows the number of instances solved for each preprocessing method used. 
As observed, the preprocessing method that preserves equivalence, {\tt \#equiv}, is not particularly effective, as it provides only a 
marginal improvement in the performance of \dfour. 
Specifically, it enables the solver to handle only 8 additional instances compared to the version without preprocessing.

\begin{figure*}[t]
    \centering
    \begin{subfigure}[t]{0.42\textwidth}
        \centering
        \includegraphics[width=\textwidth]{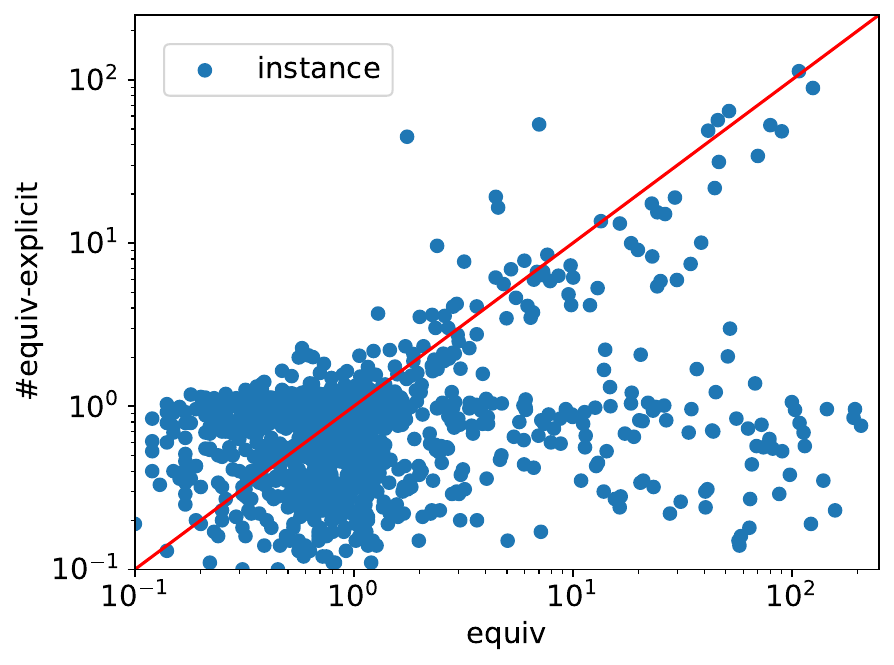}
        \caption{\label{fig:uniform} Uniform Sampling.}         
    \end{subfigure}
    \hfill
    \begin{subfigure}[t]{0.42\textwidth}
        \centering
        \includegraphics[width=\textwidth]{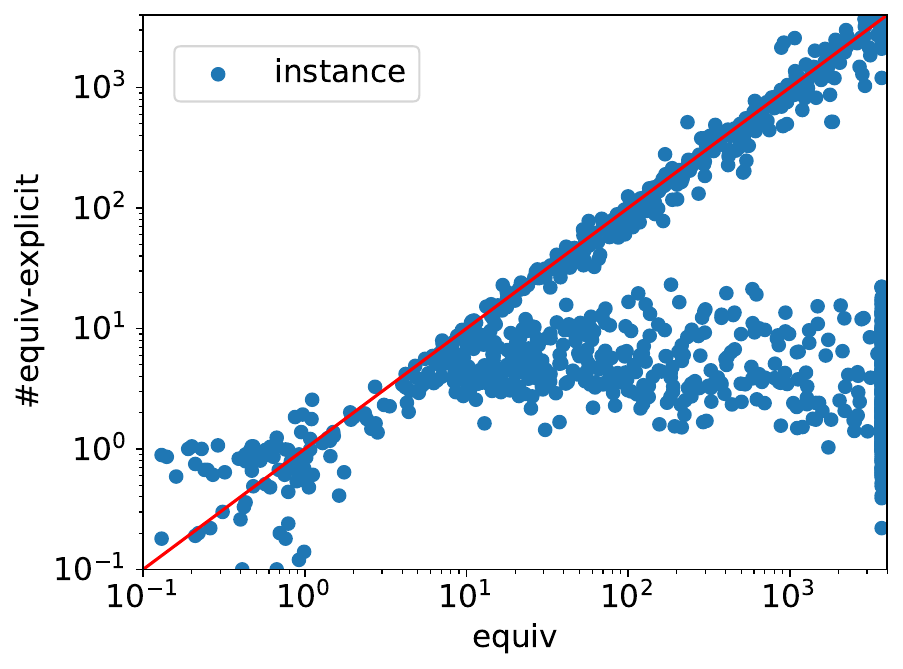}
        \caption{\label{fig:uniformSeed} Uniform Sampling with a Seed.}
    \end{subfigure}
    \hfill
    \begin{subfigure}[t]{0.42\textwidth}
        \centering
        \includegraphics[width=\textwidth]{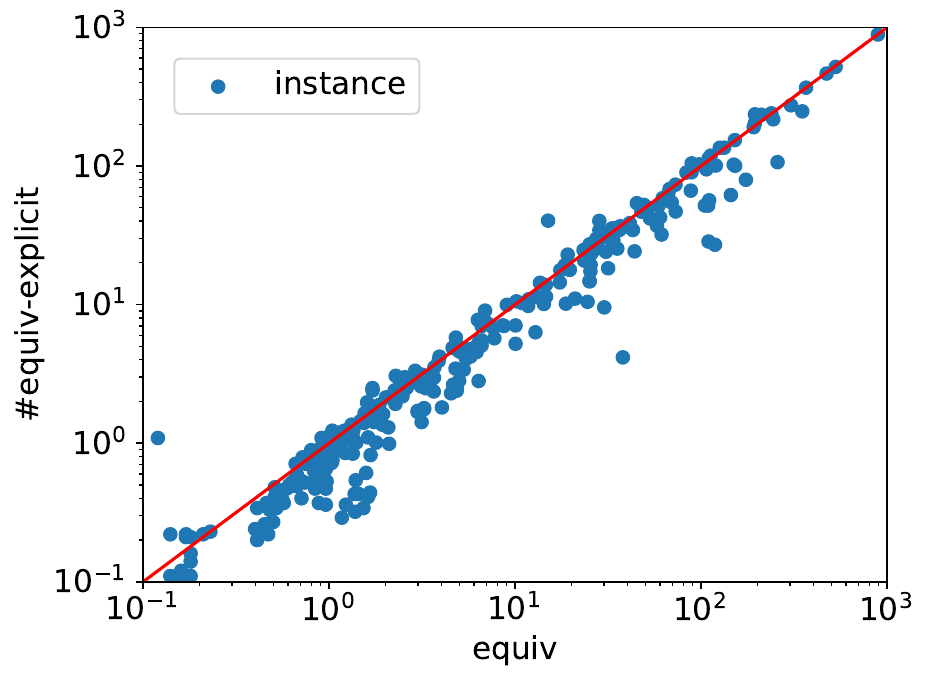}
        \caption{\label{fig:enumeration} Enumeration.}
    \end{subfigure}
    \caption{\label{fig:xp} Comparison of CPU time for the tasks under consideration: uniform sampling, uniform sampling with a seed, and model enumeration. 
    The \dDNNF{} representations analyzed were obtained using the preprocessing methods {\tt equiv} and {\tt \#equiv-explicit} for uniform sampling and enumeration, 
    and {\tt equiv} and {\tt \#equiv-explicit-ordered} for uniform sampling with a seed.}
\end{figure*}

{\tt \#equiv-explicit-ordered} substantially enhances the performance of \dfour, enabling it to handle 47 more 
instances than the compiler without preprocessing.
One of the key advantages of this method is that it supports answering direct access queries.
Next, let us examine the preprocessing method {\tt \#equiv-explicit}. 
Notably, this method is able to solve 8 more instances than {\tt \#equiv-explicit-ordered}, demonstrating its enhanced effectiveness.
Its flexibility lies in the fact that it does not require constraints on the ordering of variable assignments, 
making it particularly suitable for tasks such as model enumeration and uniform sampling.

Now, let us analyze the impact of preprocessing on the efficiency of answering the targeted queries. 
For uniform sampling and model enumeration queries, we focus the \dDNNF{} representations produced after applying the preprocessing methods 
{\tt equiv} and {\tt \#equiv-explicit}.  
To assess the direct access query, we consider uniform sampling with a seed, which leverages the direct access query.
For each instance and each query, a timeout of 3600 seconds and a memory limit of 32 GiB were imposed.  
It is important to note that this runtime does not include the compilation time.

Figure~\ref{fig:xp} displays three scatter plots comparing the running times of various \dDNNF{} representations across the three queries.
Each data point corresponds to an instance, with the x-axis showing the time (in seconds) required to solve it using the \dDNNF{} 
formula obtained after applying the {\tt equiv} preprocessing, and the y-axis showing the time needed when using the 
\dDNNF{} derived from preprocessings that eliminate defined variables.  
The experimental results clearly demonstrate that queries can be answered more efficiently using \dDNNF{} representations generated 
from formulas preprocessed with the {\tt \#equiv-explicit} based methods.  
Furthermore, the figures highlight instances where the {\tt \#equiv-explicit} based preprocessings achieve speeds up to an order of magnitude 
faster than the {\tt equiv} preprocessing, irrespective of the query type.

We evaluate the impact of {\tt \#equiv-explicit} over {\tt equiv} on uniform sampling runtimes (size 10,000).  
Figure~\ref{fig:uniform} illustrates this analysis, which includes 1,255 instances from the 1,425 assessed in the previous experiment.  
Instances excluded are 51 that failed to compile with {\tt equiv}, 4 with {\tt \#equiv-explicit}, and 115 where neither preprocessing method succeeded.  
As shown in the figure, {\tt \#equiv-explicit} generally outperforms {\tt equiv} except for instances solvable in under 2 seconds.  
Most points lie above the diagonal, indicating faster uniform sampling queries when using {\tt \#equiv-explicit}.  
Moreover, memory limits were reached in 15 instances with {\tt equiv}, compared to only 1 instance with {\tt \#equiv-explicit},  
primarily due to the larger size of the \dDNNF{} generated under {\tt equiv}.  

For model enumeration queries, we consider instances with at least $10^4$ models but fewer than $10^9$ models, 
resulting in 367 instances solved by both approaches.  
Figure~\ref{fig:enumeration} highlights the benefits of using {\tt \#equiv-explicit}, which generally provides faster enumeration.  
However, the performance gains are less pronounced than for uniform sampling queries, as enumeration still requires generating all models, 
even when the \dDNNF{} simplifies to $\top$.

To conclude, we examine the deterministic uniform sampling query (size 10,000) results in Figure~\ref{fig:uniformSeed}.
This analysis considers 1,254 out of the 1,425 instances from the previous experiment.  
Instances excluded include 44 that failed to compile with {\tt equiv}, 4 with {\tt \#equiv-explicit-ordered}, and 122 where neither method was successful.  
The \dDNNF{} produced by {\tt \#equiv-explicit-ordered} improves the manageability of uniform sampling with a seed, 
solving more instances and often one order of magnitude faster.  
Furthermore, {\tt \#equiv-explicit-ordered} significantly reduces timeouts, with only 173 compared to 391 for {\tt equiv}.

\iffalse{}
Finally, let us compare the two explicit versions of the preprocessing with the method {\tt \#equiv-implicit}, 
which adopts a different approach by disregarding the computational constraints associated with a compatible evaluation function $f$.  
%
The goal here is to assess the trade-offs between using explicit and implicit definitions, particularly in terms 
of compilation time and the computational cost of query processing.  
%
Unlike the explicit methods, which rely on explicitly defined variables and ensure that $f$ is computable in polynomial time, 
the implicit approach requires evaluating $f$ using an NP oracle, which can be expensive.  
%
Although the implicit method solves more instances, the performance gains are not substantial enough to outweigh the additional 
computational cost involved in invoking the NP oracle.  
%
As shown in the figure, while the {\tt \#equiv-implicit} approach solves more instances than the explicit methods, the difference is 
15 more instances compared to {\tt \#equiv-explicit} and 23 more compared to {\tt \#equiv-explicit-ordered}.  
%
Thus, although the implicit approach offers some advantages, the explicit methods remain more efficient in practice, 
especially given the computational overhead required to use the NP oracle.
\fi

\section{Conclusion and Perspectives\label{sec:conclusion}}
We have investigated the potential of using preprocessing methods to enhance the efficiency of direct access, uniform sampling, and model enumeration queries. 
We show that, except for the preprocessing method that maintains equivalence, current state-of-the-art preprocessing techniques are unsuitable
for these queries as they lead to incorrect results. 
However, we demonstrate that preprocessing techniques preserving the number of models can be effectively employed if information about the eliminated 
variables is retained. 
Our experimental evaluations clearly show that employing such preprocessing methods in practice is advantageous.

As a future direction, we intend to assess the preprocessing methods discussed in this paper on other advanced sampling 
techniques~\cite{SinclairJ89,JordanGJS99,ChakrabortyMV13}. 
Since these approaches do not necessarily depend on model counting, different types of preprocessing methods might be leveraged. 
Another area for improvement involves exploring alternative kinds of explicit definitions. 
In this paper, we concentrated on explicit definitions expressed as equivalences, AND-gates, OR-gates, and XOR-gates, 
but it may be beneficial to explore other types of definitions, such as NOR and NAND gates.

\begin{credits}
\subsubsection{\ackname} We thank the reviewers for their insightful comments and constructive suggestions, which helped improve the quality of this paper. This work has benefited from the support of the AI Chair EXPEKCTATION (ANR-19-CHIA-0005-01) of the French National Research Agency (ANR).
\end{credits}

\bibliographystyle{splncs04}
\bibliography{biblio}

@inproceedings{Darwiche02a,
  author    = {Adnan Darwiche},
  title     = {A Compiler for Deterministic, Decomposable Negation Normal Form},
  booktitle = {Proceedings of {AAAI}'02},
  pages     = {627--634},
  year      = {2002}
}

@inproceedings{LagniezL24,
  author    = {Jean{-}Marie Lagniez and
               Emmanuel Lonca},
  title     = {Leveraging Decision-DNNF Compilation for Enumerating Disjoint Partial
               Models},
  booktitle = {{KR}'24},
  year      = {2024}
}

@inproceedings{SoosM24,
  author    = {Mate Soos and
               Kuldeep S. Meel},
  title     = {Engineering an Efficient Preprocessor for Model Counting},
  booktitle = {{DAC}'24},
  pages     = {108:1--108:6},
  year      = {2024}
}

@inproceedings{Darwiche23,
  author    = {Adnan Darwiche},
  title     = {Logic for Explainable {AI}},
  booktitle = {{LICS}'23},
  pages     = {1--11},
  year      = {2023}
}

@article{FikesN71,
  author  = {Richard Fikes and
             Nils J. Nilsson},
  title   = {{STRIPS:} {A} New Approach to the Application of Theorem Proving to
             Problem Solving},
  journal = {Artif. Intell.},
  pages   = {189--208},
  year    = {1971}
}

@book{AbiteboulHV95,
  author    = {Serge Abiteboul and
               Richard Hull and
               Victor Vianu},
  title     = {Foundations of Databases},
  publisher = {Addison-Wesley},
  year      = {1995}
}

@inproceedings{ChakrabortyMV13,
  author    = {Supratik Chakraborty and
               Kuldeep S. Meel and
               Moshe Y. Vardi},
  title     = {A Scalable and Nearly Uniform Generator of {SAT} Witnesses},
  booktitle = {Proceedings of {CAV}'13},
  pages     = {608--623},
  year      = {2013}
}

@article{JordanGJS99,
  author  = {Michael I. Jordan and
             Zoubin Ghahramani and
             Tommi S. Jaakkola and
             Lawrence K. Saul},
  title   = {An Introduction to Variational Methods for Graphical Models},
  journal = {Mach. Learn.},
  volume  = {37},
  number  = {2},
  pages   = {183--233},
  year    = {1999}
}

@article{SinclairJ89,
  author  = {Alistair Sinclair and
             Mark Jerrum},
  title   = {Approximate Counting, Uniform Generation and Rapidly Mixing Markov
             Chains},
  journal = {Inf. Comput.},
  volume  = {82},
  number  = {1},
  pages   = {93--133},
  year    = {1989}
}

@inproceedings{LagniezM23,
  author    = {Jean{-}Marie Lagniez and
               Pierre Marquis},
  title     = {Boosting Definability Bipartition Computation Using {SAT} Witnesses},
  booktitle = {Proceedings of {JELIA}'23},
  pages     = {697--711},
  year      = {2023}
}

@article{JerrumVV86,
  author  = {Mark Jerrum and
             Leslie G. Valiant and
             Vijay V. Vazirani},
  title   = {Random Generation of Combinatorial Structures from a Uniform Distribution},
  journal = {Theor. Comput. Sci.},
  volume  = {43},
  pages   = {169--188},
  year    = {1986}
}

@inproceedings{PietteHS08,
  author    = {C{\'{e}}dric Piette and
               Youssef Hamadi and
               Lakhdar Sais},
  title     = {Vivifying Propositional Clausal Formulae},
  booktitle = {Proceedings of {ECAI}'08},
  pages     = {525--529},
  year      = {2008}
}

@inproceedings{FargierM06,
  author    = {H{\'{e}}l{\`{e}}ne Fargier and
               Pierre Marquis},
  title     = {On the Use of Partially Ordered Decision Graphs in Knowledge Compilation
               and Quantified Boolean Formulae},
  booktitle = {Proceedings of {AAAI}'06},
  pages     = {42--47},
  year      = {2006}
}

@article{CarmeliTGKR23,
  author  = {Nofar Carmeli and
             Nikolaos Tziavelis and
             Wolfgang Gatterbauer and
             Benny Kimelfeld and
             Mirek Riedewald},
  title   = {Tractable Orders for Direct Access to Ranked Answers of Conjunctive
             Queries},
  journal = {{ACM} Trans. Database Syst.},
  volume  = {48},
  number  = {1},
  pages   = {1:1--1:45},
  year    = {2023}
}

@article{BaganDGO08,
  author  = {Guillaume Bagan and
             Arnaud Durand and
             Etienne Grandjean and
             Fr{\'{e}}d{\'{e}}ric Olive},
  title   = {Computing the jth solution of a first-order query},
  journal = {{RAIRO} Theor. Informatics Appl.},
  volume  = {42},
  number  = {1},
  pages   = {147--164},
  year    = {2008}
}

@inproceedings{LagniezM17,
  author    = {Jean{-}Marie Lagniez and
               Pierre Marquis},
  title     = {An Improved Decision-DNNF Compiler},
  booktitle = {Proceedings of {IJCAI}'17},
  pages     = {667--673},
  year      = {2017}
}

@inproceedings{MuiseMBH12,
  author    = {Christian J. Muise and
               Sheila A. McIlraith and
               J. Christopher Beck and
               Eric I. Hsu},
  title     = {Dsharp: Fast d-DNNF Compilation with sharpSAT},
  booktitle = {Advances in Artificial Intelligence - 25th Canadian Conference on
               Artificial Intelligence, Canadian {AI} 2012, Toronto, ON, Canada,
               May 28-30, 2012. Proceedings},
  pages     = {356--361},
  year      = {2012}
}

@inproceedings{BringmannCM22,
  author    = {Karl Bringmann and
               Nofar Carmeli and
               Stefan Mengel},
  title     = {Tight Fine-Grained Bounds for Direct Access on Join Queries},
  booktitle = {Proceedings of {PODS}'22},
  pages     = {427--436},
  year      = {2022}
}

@inproceedings{Darwiche04,
  author    = {Adnan Darwiche},
  title     = {New Advances in Compiling {CNF} into Decomposable Negation Normal
               Form},
  booktitle = {Proceedings of {ECAI}'04},
  pages     = {328--332},
  year      = {2004}
}

@article{HeuleSB17,
  author  = {Marijn J. H. Heule and
             Martina Seidl and
             Armin Biere},
  title   = {Solution Validation and Extraction for {QBF} Preprocessing},
  journal = {J. Autom. Reason.},
  volume  = {58},
  number  = {1},
  pages   = {97--125},
  year    = {2017}
}

@incollection{BiereJK21,
  author    = {Armin Biere and
               Matti J{\"{a}}rvisalo and
               Benjamin Kiesl},
  title     = {Preprocessing in {SAT} Solving},
  booktitle = {Handbook of Satisfiability - Second Edition},
  volume    = {336},
  pages     = {391--435},
  publisher = {{IOS} Press},
  year      = {2021}
}

@article{LagniezM17b,
  author  = {Jean{-}Marie Lagniez and
             Pierre Marquis},
  title   = {On Preprocessing Techniques and Their Impact on Propositional Model
             Counting},
  journal = {J. Autom. Reason.},
  volume  = {58},
  number  = {4},
  pages   = {413--481},
  year    = {2017}
}

@inproceedings{SoosM22,
  author    = {Mate Soos and
               Kuldeep S. Meel},
  title     = {Arjun: An Efficient Independent Support Computation Technique and
               its Applications to Counting and Sampling},
  booktitle = {Proceedings of {ICCAD}'22},
  pages     = {71:1--71:9},
  year      = {2022}
}

@inproceedings{SubbarayanP04,
  author    = {Sathiamoorthy Subbarayan and
               Dhiraj K. Pradhan},
  title     = {NiVER: Non Increasing Variable Elimination Resolution for Preprocessing
               {SAT} instances},
  booktitle = {Proceedings of {SAT}'04},
  year      = {2004}
}

@inproceedings{ZhangMMM01,
  author    = {Lintao Zhang and
               Conor F. Madigan and
               Matthew W. Moskewicz and
               Sharad Malik},
  title     = {Efficient Conflict Driven Learning in Boolean Satisfiability Solver},
  booktitle = {Proceedings of {ICCAD}'01},
  pages     = {279--285},
  year      = {2001}
}

@inproceedings{LagniezM14,
  author    = {Jean{-}Marie Lagniez and
               Pierre Marquis},
  title     = {Preprocessing for Propositional Model Counting},
  booktitle = {Proceedings of {AAAI}'14},
  pages     = {2688--2694},
  year      = {2014}
}

@inproceedings{JarvisaloBH10,
  author    = {Matti J{\"{a}}rvisalo and
               Armin Biere and
               Marijn Heule},
  title     = {Blocked Clause Elimination},
  booktitle = {Proceedings of {TACAS}'10},
  volume    = {6015},
  pages     = {129--144},
  year      = {2010}
}

@inproceedings{LagniezLM16,
  author    = {Jean{-}Marie Lagniez and
               Emmanuel Lonca and
               Pierre Marquis},
  title     = {Improving Model Counting by Leveraging Definability},
  booktitle = {Proceedings of {IJCAI}'16},
  pages     = {751--757},
  year      = {2016}
}

@article{LagniezLM20,
  author  = {Jean{-}Marie Lagniez and
             Emmanuel Lonca and
             Pierre Marquis},
  title   = {Definability for model counting},
  journal = {Artif. Intell.},
  volume  = {281},
  pages   = {103229},
  year    = {2020}
}

@inproceedings{KieselE23,
  author    = {Rafael Kiesel and
               Thomas Eiter},
  title     = {Knowledge Compilation and More with SharpSAT-TD},
  booktitle = {Proceedings of {KR}'23},
  pages     = {406--416},
  year      = {2023}
}

@inproceedings{SharmaGRM18,
  author    = {Shubham Sharma and
               Rahul Gupta and
               Subhajit Roy and
               Kuldeep S. Meel},
  title     = {Knowledge Compilation meets Uniform Sampling},
  booktitle = {Proceedings of {LPAR}'22},
  volume    = {57},
  pages     = {620--636},
  year      = {2018}
}

\end{document}